\documentclass{article}

\usepackage{hyperref}
\usepackage{fullpage}

\usepackage{amsmath}
\usepackage{amssymb}
\usepackage{mathtools}
\usepackage{amsthm}
\usepackage{enumerate}

\theoremstyle{plain}
\newtheorem{theorem}{Theorem}[section]

\newtheorem{lemma}[theorem]{Lemma}
\newtheorem{corollary}[theorem]{Corollary}
\newtheorem{definition}[theorem]{Definition}

\newtheorem{remark}[theorem]{Remark}
\newtheorem{example}[theorem]{Example}
\newtheorem{claim}[theorem]{Claim}

\usepackage[textsize=tiny]{todonotes}

\usepackage[utf8]{inputenc}
\usepackage{color}
\usepackage{parskip}
\usepackage[ruled,vlined,noend]{algorithm2e}
\usepackage[noend]{algorithmic}

\usepackage{xfrac}
\usepackage[inline]{enumitem}
\usepackage{bbm}
\usepackage{xifthen}

\usepackage{xcolor}
\usepackage[textsize=tiny]{todonotes}

\newcommand\numberthis{\addtocounter{equation}{1}\tag{\theequation}}

\DeclarePairedDelimiter{\set}{\{}{\}}

\newcommand{\R}{\mathbb{R}}
\newcommand{\N}{\mathbb{N}}
\newcommand{\E}{\mathop{\mathbb{E}}}
\newcommand{\X}{\mathcal{X}}
\newcommand{\Y}{\mathcal{Y}}
\newcommand{\AAA}{\mathcal{A}}

\newcommand{\mean}[1]{\mathop{\mathbb{E}}_{#1}}
\newcommand{\alg}{\mathcal{A}}

\newcommand{\sample}{S}
\newcommand{\measure}{\mu}
\newcommand{\symmix}[2]{\gamma_{#1}^{#2}}

\newcommand{\abs}[1]{\left|#1\right|}
\newcommand{\norm}[1]{\left\lVert#1\right\rVert}
\newcommand{\bigO}[1]{\mathcal{O}\left(#1\right)}
\newcommand{\tildeO}[1]{\widetilde{\mathcal{O}}\left(#1\right)}

\newcommand{\tv}[2]{\left\lVert#1 - #2\right\rVert_{\textup{\tiny\textsf{TV}}}}
\newcommand{\tvleft}{\Vert}
\newcommand{\tvright}{{\Vert}_{\textup{\tiny\textsf{TV}}}}

\newcommand{\analyst}{\mathcal{A}}
\newcommand{\game}{AG}

\newcommand{\uf}[1]{\frac{1}{#1}}
\newcommand{\adversary}{\mathbb{A}}
\newcommand{\mechanism}{\mathbb{M}}
\newcommand{\mixing}{\psi}

\newcommand{\uniform}[1]{U\left(#1\right)}
\newcommand{\bernulli}[1]{{\rm Ber}\left(#1\right)}
\newcommand{\seed}{x^*}
\newcommand{\histo}{\mathcal{H}}

\newcommand{\indicator}[1]{\mathbbm{1}\left[#1\right]}
\usepackage{hyperref}
\newcommand{\potential}{g}

\newcommand{\remove}[1]{}
\newcommand{\TSremove}[1]{}

\usepackage[round]{natbib}

\begin{document}
\title{Adaptive Data Analysis with Correlated Observations}
\author{
Aryeh Kontorovich\thanks{Ben-Gurion University.}
\and
Menachem Sadigurschi\thanks{Ben-Gurion University. Partially supported by the Israel Science Foundation (grant 1871/19).}
\and
Uri Stemmer\thanks{Tel Aviv University and Google Research. Partially supported by the Israel Science Foundation (grant 1871/19)
and by Len Blavatnik and the Blavatnik Family foundation.}
}

\date{January 21, 2022}
\maketitle

\begin{abstract}
The vast majority of the work on adaptive data analysis focuses on the case where the samples in the dataset are independent. Several approaches and tools have been successfully applied in this context, such as {\em differential privacy}, {\em max-information}, {\em compression arguments}, and more. The situation is far less well-understood without the independence assumption. 

We embark on a systematic study of the possibilities of adaptive data analysis with correlated observations. First, we show that, in some cases, differential privacy guarantees generalization even when there are dependencies within the sample, which we quantify using a notion we call {\em Gibbs-dependence}. We complement this result with a tight negative example.
Second, we show that the connection between transcript-compression and adaptive data analysis can be extended to the non-iid setting.

\end{abstract}

\section{Introduction}\label{sec:introduction}

Statistical validity is a well known crucial aspect of modern science.
In the past several years, the natural science and social science communities
have come to realize that such validity was not in fact preserved
in numerous peer-reviewed and widely cited studies, leading to many false discoveries. Known as the {\em replication crisis}, 
this phenomenon threatens to undermine the very basis for the public's trust in science.

One of the main explanations for the prevalence of false discovery arises from the inherent {\em adaptivity} in the process of data
analysis. To illustrate this issue, consider a data analyst interested
in testing a specific research hypothesis. The analyst acquires relevant data, evaluates
the hypothesis, and (say) learns that it is false. Based on the findings, the analyst now
decides on a second hypothesis to be tested, and evaluates it on the {\em same data} (acquiring
fresh data might be too expensive or even impossible). That is, the analyst chooses the
hypotheses {\em adaptively}, where this choice depends on previous interactions with the data.
As a result, the findings are no longer supported by classical statistical theory, which
assumes that the tested hypotheses are fixed before the data is gathered, and the analyst runs the risk of overfitting to the data.

Before presenting our new results, we make the setting explicit. We give here the formulation presented by \citet{dwork2015preserving}. We consider a two-player game between a mechanism $\mechanism$ and an adversary $\adversary$, defined as follows (see Section~\ref{sec:preliminaries} for precise definitions).
\begin{enumerate}[topsep=0pt,itemsep=0pt,partopsep=0ex,parsep=2pt]
    \item The adversary $\adversary$ fixes a measure $\measure$ over $\X^n$ (satisfying some conditions).
    \item The mechanism $\mechanism$ obtains a sample $S\sim\measure$ containing $n$ (possibly correlated) observations.
    \item For $k$ rounds $j=1,2,\dots,k$:
    \begin{itemize}
        \item The adversary chooses a {\em query} $h_j:\X\rightarrow\{0,1\}$, possibly as a function of all previous answers given by the mechanism.
        \item The mechanism obtains $h_j$ and responds with an answer $z_j\in\R$, which is given to $\adversary$.
    \end{itemize}
\end{enumerate}

We say that $\mechanism$ is {\em $(\alpha,\beta)$-empirically-accurate} if with probability at least $1-\beta$ for every $j$ it holds that $|z_j-h_j(S)|\leq\alpha$, where $h_j(S)=\frac{1}{n}\sum_{x\in S}h_j(x)$ is the empirical average of $h_j$ on the sample $S$. We say that $\mechanism$ is {\em $(\alpha,\beta)$-statistically-accurate} if with probability at least $1-\beta$ for every $j$ it holds that $|z_j-h_j(\mu)|\leq\alpha$, where $h_j(\mu)=\E_{T\sim \mu}\left[h_j(T)\right]=\E_{T\sim \mu}\left[\frac{1}{n}\sum_{x\in T}h_j(x)\right]$ is the ``true'' value of the query $h_j$ on the underlying distribution $\mu$. Our goal is to design mechanisms $\mechanism$ providing statistical-accuracy.

Starting from \citep{dwork2015reusable,dwork2015generalization}, it has been demonstrated that  
various
notions of {\em algorithmic stability},
and in particular 
{\em differential privacy (DP)}~\citep{dwork2006calibrating},
allow for methods which maintain statistical validity under the adaptive setting.
The vast majority of the works in this vein, however,  strongly rely on the assumption that the data is sampled in an i.i.d.\ fashion.
This scenario excludes some natural and essential problems in learning theory
such as Markov chains, active learning, and autoregressive models
\citep{
kontram06,jap/1421763330,
kontorovich2017concentration,
settles2009active,hanneke2014theory,
10.1162/00335530151144131}.

A notable exception is a stability notion introduced by \citet{DBLP:journals/corr/BassilyF16}, called {\em typical-stability}.
This beautiful and natural notion has the advantage that, under some conditions on the underlying distribution, it can guarantee statistical validity even for non-i.i.d.\ settings. However, one downside of the results of \citet{DBLP:journals/corr/BassilyF16} is that they do not recover the i.i.d.\ generalization bounds in the limiting regime where the dependencies decay to zero. In particular, in the i.i.d.\ setting, it is possible to efficiently answer $\tilde{O}(n^2)$ adaptive queries given a sample of size $n$. In contrast, the results of \citet{DBLP:journals/corr/BassilyF16} only allow to answer $\tilde{O}(n)$ adaptive queries, {\em even if the dependencies in the data decay to zero}. 
Bridging this gap is one of the main motivations for our work.

\subsection{Our Contributions}
\label{sec:main-results}
We reestablish the baseline for adaptive data analysis with correlated observations.
Our first contribution is to extend existing generalization results for differential privacy from the i.i.d.\ setting to the correlated setting.
To that end, we 
introduce
a notion we call \emph{Gibbs dependence} to quantify the  dependencies between 
the 
covariates
of a given 
joint
distribution. 
We complement this result with a tight negative example. 
Our second contribution is 
to extend the connection between transcript-compression and adaptive data analysis also to the non-iid setting.
Finally, we demonstrate an application of our results for when the underlying measure can be described as a Markov chain.

\subsubsection{Gibbs Dependence}
\label{subsec:into-gd-dp}

We extend the connection between differential privacy and generalization to the case where the observations are correlated. We quantify the correlations in the data using a new notion, called \emph{Gibbs dependence}, which is closely related to the classical 
\emph{Dobrushin} interdependence coefficient \citep{kontorovich2017concentration,levin2017markov}.
Intuitively, a measure which has $\mixing$-Gibbs dependency
is such that knowledge about almost the entire sample
does not provide too much information about the remaining portion.
Formally,

\begin{definition}
  For a probability measure $\measure$
  over a product space $\X^n$,
  define
    \[\mixing(\measure)
     = 
     \sup_{x\in\X^n}
     \E_{i\sim[n]}
     \tv{\measure_i(\cdot)}{\measure_i(\cdot\mid x^{-i})},
     \]    
  where $\measure_i(\cdot)$ is the $i^{th}$ marginal measure 
  and $\measure_i(\cdot\mid x^{-i})$ is the $i${th} marginal measure
  conditioned on all the coordinates other than $i$
  (given some $n$-tuple $x$). 
  Given $\mixing$
  we say the $\measure$ has $\mixing$-Gibbs dependence
  if $\mixing(\measure) \leq \mixing$. 
  For a series of probability measures
  $\{\measure_n\}_{n\in\N}$,
  we say that the series has strong-Gibbs dependence
  if
  $\mixing(\measure_n) \xrightarrow[n \to \infty]{} 0$.
\end{definition}

\begin{example}[Product Measures]
\label{ex:1}
A probability measure $\measure$ has $\mixing(\measure)=0$ if and only if it is a product measure.
This is since, for a product measure, we get that  for every $i\in [n]$ $\measure_i(\cdot) = \measure_i(\cdot \mid x^{-i})$ 
and hence 
$\sup{x}\tv{\measure_i(\cdot)}{\measure_i(\cdot\mid x^{-i})} = 0$, 
which means 
$\mixing(\measure)=0$. 
From the other side when $\mixing(\measure)=0$, 
as $\tv{\measure_i(\cdot)}{\measure_i(\cdot\mid x^{-i})}$ 
is non-negative, we get that $\tv{\measure_i(\cdot)}{\measure_i(\cdot\mid x^{-i})} = 0$, meaning that $\measure_i(\cdot) \equiv \measure_i(\cdot\mid x^{-i})$ for every possible $x$. 
This implies that $\measure$ is a product measure (see Appendix~\ref{apn:prod} for more details).
\end{example}

\begin{example}[Markov Random Fields]
\label{ex:4}
Let $G = ([n], V)$ be a graph on the indices set, and 
recall that $\measure$ is a \emph{Markov Random Field} w.r.t.\ $G$ 
if for every $i \in [n]$
\[
  \measure(x_i\mid x^{-i}) = \measure(x_i\mid x_{\Gamma(i)}),
\]
where $\Gamma(i)$ is the set of neighbors of $i$ on $G$.
This enables modeling undirected dependencies. 
Such dependencies are common on computational biology \citep{banf2017enhancing,10.1093/bioinformatics/btm129}, computer vision \citep{li2009markov,blake2011markov} and physics \citep{rue2005gaussian}. 
Markov Random Fields are closely related to our notion of Gibbs dependence.
We elaborate on this connection in Section~\ref{sec:umc}.
\end{example}

A naive way for leveraging our notion of Gibbs dependence would be to ``union bound'' the correlations across the $n$ different coordinates. Specifically, one could show that if $\measure$ has Gibbs-dependence $\psi$ then $\tv{\measure}{\measure^{*}}\leq n\psi$, where $\measure^*$ is the product distribution in which every coordinate is sampled independently from the corresponding marginal distribution in $\measure$. Thus, if $\psi\ll\frac{1}{n}$ then one could argue about generalization w.r.t.\ $\measure$ by applying existing generalization bounds w.r.t.\ $\measure^*$ in the independent case (since in this regime we have $\tv{\measure}{\measure^{*}}\ll1$). This argument, however, only works when the dependencies in $\measure$ are {\em very weak} (i.e., when $\psi\ll\frac{1}{n}$). Our main result is to show that differential privacy still provides generalization even if $\psi$ is much larger, e.g., a constant independent of the sample size $n$. 
Specifically, 

\begin{theorem}
  \label{thm:main-result-dp}
  Let $\mechanism$ be an $(\varepsilon,\delta)$-differentially-private mechanism
  which is $(\alpha,\beta)$-empirically-accurate for $k$ rounds given $n$ samples. 
  If $n\geq \frac{\log(2k\varepsilon/\delta)}{\varepsilon^2}$, then $\mechanism$ is also
  $(\alpha+10\varepsilon+2\mixing,\beta + \frac{\delta}{\varepsilon})$-statistically-accurate.
\end{theorem}

\begin{remark}
For the case when $\mixing$ is zero, and hence $\measure$ is a product measure (see Example~\ref{ex:1}), Theorem~\ref{thm:main-result-dp} recovers the optimal results achieved by differential privacy for i.i.d.\ samples \citep{dwork2015preserving,bassily2016algorithmic}. Thus, Theorem~\ref{thm:main-result-dp} generalizes the connection between differential privacy and generalization to the correlated setting.
\end{remark}

Intuitively, the above theorem states that if the underlying distribution has Gibbs-dependence $\psi$ then the additional generalization error incurred by DP algorithms (compared to the iid setting) is at most $O(\psi)$. We complement this result with a tight negative example showing that there exist a distribution $\measure$ with Gibbs-dependence $\psi$ and a DP algorithm $\AAA$ that obtains generalization error $\Omega(\psi)$. This means that, in terms of the Gibbs-dependence, our result is tight.

By applying Theorem~\ref{thm:main-result-dp} with a known DP mechanism for answering queries while providing empirical accuracy, we get the following corollary.

\begin{corollary}\label{cor:main}
There is a computationally efficient mechanism $\mechanism$ that is $(\alpha+2\psi,\beta)$-statistically-accurate for $k$ adaptively chosen queries given a sample (an $n$-tuple) from an underlying measure with Gibbs-dependency $\psi$ provided that  
$
n\geq \tilde{O}\left( \frac{\sqrt{k}}{\alpha^2}\log\frac{1}{\beta} \right).
$
\end{corollary}

This generalizes the state-of-the-art bounds for the i.i.d.\ setting, where $\psi=0$. In particular, Corollary~\ref{cor:main} shows that mild dependencies in the data, say $\psi=\alpha$, come {\em for free} in terms of the achievable bounds for adaptive data analysis. We emphasize that $\psi=\alpha$ captures non-negligible dependencies. In particular, $\alpha$ could be constant, independent of the sample size $n$.

\subsubsection{Transcript Compression}

The second direction we examine is that of \emph{transcript compression}.
The concept of compression is a central idea in the learning literature. It is both an algorithmic tool and a statistical tool, used both for designing learning mechanisms and achieving a better understanding of the concept of generalization (see for example \citep{littlestone1986relating, NEURIPS2019_860b37e2, moran2016sample, ashtiani2020near, hanneke2021stable}).

Compression has also been used in the context of adaptive data analysis. \citet{dwork2015generalization} used the definition of \emph{bounded description length} (referred to here as \emph{transcript compression}) to present an algorithm that is able to adaptively answer queries when the data is i.i.d.\ sampled. 
Our contribution here is in generalizing this idea by showing that the same definition, when used in the right setting, allows maintaining adaptive accuracy even when the distribution includes dependencies.

Following the approach of \citet{DBLP:journals/corr/BassilyF16}, we aim to provide the following guarantee: 
As long as the analyst chooses functions which, in the non-adaptive setting, are concentrated around their expected value, then the answers given by the mechanism should be accurate. Intuitively, the idea is that functions with large variance are hard to approximate even in the non-adaptive setting, and hence, we should not require our mechanism to approximate them well in the adaptive setting.

This is formalized as follows. For every query $q$ and every 
$\measure$, we write $\gamma(q,\measure,\delta)$ to denote the length of a confidence interval around the expectation of $q$ with confidence level $(1-\delta)$. That is, $\gamma(q,\measure,\delta)$ is such that when sampling $T\sim\measure$, with probability at least $(1-\delta)$ it holds that $q(T)$ is within $\gamma(q,\measure,\delta)$ from its expectation. We obtain the following theorem (see Section~\ref{sec:compression} for a precise statement).

\begin{theorem}[informal]\label{thm:comp_informal}
Fix $\alpha,\delta>0$. 
There exists a computationally efficient mechanism with the following properties. The mechanism obtains a sample (an $n$-tuple) from some unknown underlying distribution $\measure$. Then, for $k$ rounds $i=1,2,\dots,k$, the mechanism obtains a query $q_i$ and responds with an answer $a_i$ such that 
$$
\Pr[\exists i \text{ s.t.\ } |a_i-q_i(\measure)|>\alpha+\gamma(q_i,\measure,\delta)]\leq \delta\cdot k \cdot 2^{k\cdot\log\frac{1}{\alpha}}.
$$
\end{theorem}

In particular, as long as the adversary poses queries $q_i$ such that $\gamma(q_i,\measure,\delta)\leq\alpha$, the mechanism from Theorem~\ref{thm:comp_informal} guarantees that all of its answers are $2\alpha$-accurate, with probability at least $1-\delta\cdot k \cdot 2^{k\cdot\log\frac{1}{\alpha}}$. In order for such a statement to be  meaningful, we want to assert that $\delta\ll\frac{1}{k} \cdot 2^{-k\cdot\log\frac{1}{\alpha}}$. This is easily obtained in many settings of interest by taking the sample size $n$ to be big enough. For example, for sub-Gaussian or sub-exponential queries, we would get that $\delta$ vanishes exponentially with $n$, and hence, for large enough $n$ we would get that $\delta\ll\frac{1}{k} \cdot 2^{-k\cdot\log\frac{1}{\alpha}}$.

\subsection{Comparison to \citet{DBLP:journals/corr/BassilyF16}}
\citet{DBLP:journals/corr/BassilyF16} also studied the problem of adaptive data analysis with correlated observations. Our results differ from theirs on the following points.
\begin{enumerate}
    \item \citet{DBLP:journals/corr/BassilyF16} can answer at most $\tildeO{n}$ adaptive queries efficiently, even if the dependencies within the sample are very weak. Using our notion of Gibbs-dependency, we can answer $\tildeO{n^2}$ adaptive queries efficiently, while accommodating small (but non-negligible) dependencies.
    \item As we mentioned,  \citet{DBLP:journals/corr/BassilyF16} introduced the beautiful framework where the mechanism is required to provide accurate answers only as long as the analyst poses ``concentrated queries''. They obtained their results for this setting via a new notion they introduced, called {\em typical stability}. However, their analysis and definitions are quite complex. We show that essentially the same bounds can be obtained {\em in a significantly simpler way, using standard compression tools}. Specifically, our result in this context (Theorem~\ref{thm:comp_informal}) recovers essentially the same bounds for all types of queries considered by  \citet{DBLP:journals/corr/BassilyF16}, including bounded-sensitivity queries, subgaussian queries, and subexponential queries. In addition to being significantly simpler, our result in this context offers the following advantage: Using the results of \citet{DBLP:journals/corr/BassilyF16}, we need to know {\em in advance} the parameter controlling the ``concentration level''  of the queries that will be presented in runtime, and this parameter is used by their algorithm. In contrast, our algorithm is oblivious to this parameter, and the guarantee is that our accuracy depends on the ``concentration level'' of the given queries. Furthermore, with our algorithm, different queries throughout the execution can have different ``concentration levels'', a feature which is not directly supported by \citet{DBLP:journals/corr/BassilyF16}.
\end{enumerate}

\subsection{Other related works}

Algorithmic stability is known to be intimately connected (and, in some settings, equivalent) to
learnability~\cite{bousquet2002stability,shalev2010learnability}. Most of the existing stability notions, however, are not sufficient for our goal of adaptive learnability. For example, {\em uniform stability}, which has recently been the subject of several interesting results, is not closed under post-processing and does not yield the same type of adaptive generalization bounds as we study in this paper.~\citep{BousquetE02,Shalev-ShwartzSSS10,pmlr-v48-hardt16,FeldmanV18,FeldmanV19} 
A notable exception is {\em local statistical stability}, which was shown to be both necessary and sufficient for adaptive generalization \citep{ShenfeldL19}. However, so far, local statistical stability has not yielded new algorithmic insights.

A different line of research employs information-theoretic techniques,
whereby overfitting is prevented by bounding the amount of mutual information
between the input sample and the output hypothesis.
However, these techniques generally only guarantee generalization in expectation, rather than high probability bounds.~\cite{pmlr-v51-russo16,NIPS2017_6846,RogersRST16,RaginskyRTWX16,russo2019much,SteinkeZ20}.

The formulation of the adaptive data analysis we consider was introduced by \citet{dwork2015reusable} (in the context of i.i.d.\ sampling), and has since then been the subject of many interesting papers \citep{bassily2016algorithmic,bun2018fingerprinting,hardt2014preventing,NIPS2018_7876,ShenfeldL19,JungLN0SS20,abs-2106-10761}. The connection between differential privacy and adaptive generalization also originated from \citet{dwork2015reusable}. Interestingly, this connection has recently been repurposed for different settings, such as adversarial streaming and dynamic algorithms \citep{HassidimKMMS20,abs-2107-14527,KaplanMNS21,abs-2111-03980}.

We note that in the case of {\em non-adaptive} data analysis, learning from non-i.i.d samples is a well-known problem that has been heavily studied in various directions. This includes works on the Markovian criteria  \cite{marton_measure_1996,kontorovich2017concentration,DBLP:conf/alt/WolferK19,Juang1991HiddenMM}, as well as other criterias \cite{DBLP:conf/stoc/DaskalakisDP19,DBLP:conf/colt/DaganDDJ19}. These lines of work do not transfer, at least not in a way that we are aware of, to the adaptive setting.

\section{Preliminaries}
\label{sec:preliminaries}

Denote by $\X$ a metric space
and let $\measure$ be a probability measure on $\X^n$.
Throughout the paper, we will use $\vec{V}$ to denote vectors.
For a vector or a set $x$, we write $x_i$ to denote the $i^{th}$ element of $x$.
We will use superscript with a minus sign to denote the whole sequence besides the given index,
so $S^{-i}$ is the sequence $S$ excluding the $i^{th}$ element of $S$.
For a probability measure $\measure$ over $\X^n$ denote by
$\measure_i$ the marginal distribution over the $i^{th}$ coordinate.

Our main metric for similarity between probability measures will be the \emph{total variation distance}.

\begin{definition}[Total Variation Distance]
  Given two measures $\nu$ and $\mu$ on the same space $\Omega$,
  the {\em total variation distance}
  between them is defined as 
  $\tv{\nu}{\mu} := \sup_{A\subseteq \Omega} \abs{\nu(A) - \mu(A)},$
  where the supremum is over the Borel sets of $\Omega$. Equivalently, $\tv{\nu}{\mu} = \frac{1}{2}\sum_{a\in\Omega}\abs{\nu(a) - \mu(a)} = \frac{1}{2}\norm{\mu - \nu}_{\ell_1}.$
\end{definition}

\subsection{Preliminaries from differential privacy}

Differential privacy \citep{dwork2006calibrating} is a mathematical definition for privacy that aims to enable statistical analyses of datasets while providing strong guarantees that individual-level information does not leak. Informally, an algorithm that analyzes data satisfies differential privacy if it is robust in the sense that its outcome distribution does not depend ``too much'' on any single data point. Formally,

\begin{definition}[\citet{dwork2006calibrating}]
  Random variables $X,Y$ with the same range $\Omega$ are said to have $(\eta,\tau)$-indistinguishable distributions,
  denoted as $X \approx_{\eta,\tau} Y$,
  if for all measurable subsets $A\subseteq \Omega$
  we have \quad
  $\Pr[X\in A] \leq e^\eta \Pr[Y\in A] + \tau$
  \quad and \quad
  $\Pr[Y\in A] \leq e^\eta \Pr[X\in A] + \tau.$
\end{definition}

\begin{definition}[Differential Privacy \citep{dwork2006calibrating}]
  A randomized algorithm $\alg:\X^n\to \Y$
  is $(\varepsilon,\delta)$-differentially private
  if for every two datasets $S,S'$ which differ on a single element we have $\alg(S) \approx_{\varepsilon,\delta} \alg(S').$
\end{definition}

One of the most basic and generic tools in the literature on differential privacy is the exponential mechanism of \citet{mcsherry2007mechanism}, defined as follows. Consider a ``quality function'' $f$ that, given a dataset $S$, assigns every possible solution $a$ (coming from some predefined solution-set $A$) a real valued number, identified as the ``score'' of the solution $a$ w.r.t.\ the input dataset $S$. The goal is to privately identify a solution $a\in A$ with a high score $f(S,a)$. The mechanism itself simply picks a solution at random,
where the probability for solution $a$ is proportional to $e^{\varepsilon f(S,a)}$.
As shown by  \citet{mcsherry2007mechanism} the exponential mechanism
is $(\varepsilon,0)$-differentially private.

\subsection{Preliminaries on adaptive data analysis}
The standard formulation of adaptive data analysis is defined as a game
involving some (adversary) analyst and a query-answering mechanism. For the sake of this paper queries are {\em statistical queries}, meaning they are functions of the form $q:\X\to [0,1]$.
The goal of the mechanism is to make sure that the answers provided to the analyst
are accurate w.r.t.\ the expected value of the corresponding queries over the underlying distribution. The idea is to formalize a utility notion that holds for \emph{any} strategy of the data analyst. As a way of dealing with \emph{worst-case analysts}, the analyst is assumed to be \emph{adversarial} in that it tries to cause the mechanism to fail. If a mechanism can maintain utility against any such and \emph{adversarial} analyst, then it maintains utility against any analyst. This game is specified in Algorithm~\ref{alg:adapt-pop-game}.

\begin{algorithm}[tb]
   \caption{$\texttt{Game}(\mechanism,k,\adversary,S)$}
  \label{alg:adapt-pop-game}
\begin{algorithmic}
   \STATE {\bfseries Inputs:} Mechanism $\mechanism$, interaction length $k$, adversary $\adversary$, dataset $S$.
   \STATE The dataset $S$ is given to $\mechanism$.
   \FOR{$i\in [k]$}
   \STATE $\adversary$ picks a query $q_i$.
   \STATE The query $q_i$ is given to $\mechanism$.
   \STATE $\mechanism$ outputs an answer $a_i$.
   \STATE The answer $a_i$ is given to $\adversary$.
   \ENDFOR
\end{algorithmic}
\end{algorithm}

\begin{definition}[Adaptive Empirical Accuracy]
  A mechanism $\mechanism$ is {\em $(\alpha,\beta)$-empirically-accurate} for $k$ rounds given a dataset of size $n$,
  if for every dataset $S$ of size $n$ and every adversary $\adversary$, it holds that 
  \[
    \Pr_{\texttt{Game}(\mechanism,k,\adversary,S)}\left[
      \max_{i\in [k]} \abs{q_i(S) - a_i} > \alpha
    \right] \leq \beta,
  \]
  where $q_i(S) := \uf{|S|}\sum_{x\in S}q_i(x)$.
\end{definition}

\begin{definition}[Adaptive Statistical Accuracy]\label{def:adaptiveaccuracy}
  A mechanism $\mechanism$ is {\em $(\alpha,\beta,\mixing)$-statistically-accurate} for $k$ rounds given $n$ samples, 
  if for every distribution $\measure$ over $n$-tuples with Gibbs dependency $\mixing$, and every adversary $\adversary$, 
  it holds that 
  \[
    \Pr_{\substack{S\sim\measure\\\texttt{Game}(\mechanism,k,\adversary,S)}}\left[
      \max_{i\in [k]} \abs{q_i(\measure) - a_i} > \alpha
    \right] \leq \beta,
  \]
  where $q_i(\measure) := \mean{T\sim\measure}[q_i(T)]=\mean{T\sim\measure}\left[\uf{|T|}\sum_{x\in T}q_i(x)\right]$.
\end{definition}

\begin{remark}
\label{rem:det-is-suff}
    The above definition is stated in general form, but in fact it is sufficient to show that a mechanism $\mechanism$ exhibits the above guarantee for every \emph{deterministic} adversary $\adversary$. The reason is that for a randomized adversary one can fix the adversary's random coins and use the \emph{total probability law} in order to get the same result. 
\end{remark}

\section{Adaptive Generalization via Differential Privacy}
\label{sec:learning-with-low}

We extend the connection between differential privacy and adaptive data analysis into settings where the data is not sampled in an i.i.d.\ fashion, but rather there are some small/bounded dependencies.
We start by proving the following lemma, showing that differential privacy guarantees generalization in expectation. %
The proof of this lemma mimics the analysis of \citet{bassily2016algorithmic} for the i.i.d.\ setting. We extend the proof to the case where there are dependencies in the data, and show that we can ``pay'' for these dependencies in a way that scales with $\psi$.

\begin{lemma}[Expectation bound]
  \label{lem:dp-expect-bound}
  Let $\alg':(\X^n)^T\to 2^\X\times [T]$
  be an $(\varepsilon,\delta)$-differentially private algorithm.
  Let $\measure$ be a distribution over $\X^n$
  which has $\mixing$-Gibbs-dependence
  let $\vec{S} = (S_1,\ldots,S_T)$
  where for every $i$ $S_i\sim \measure$.
  Denote by $(h,t)$ the output of $\alg'(\vec{S})$.
  Then
  \[
    \abs{\E_{\vec{S},\alg'}\left[h(\measure)-h(S_t)\right]} \leq e^\varepsilon + T\delta + \mixing - 1.
  \]
\end{lemma}

\begin{proof}%

We consider a {\em multi sample} $\vec{S}=(S_1,\dots,S_T)$, where $S_t=(x_{t,1},\dots,x_{t,n})\sim\measure$. We calculate,

  \begin{align*}
    \E_{\vec{S}\sim\measure^T}\left[\E_{(h.t)\sim\alg'(\vec{S})}[h(S_t)]\right] 
    =&
      \E_{\vec{S}\sim\measure^T}\left[      \E_{(h.t)\sim\alg'(\vec{S})}\left[\uf{n}\sum_{i=1}^nh(x_{t,i})\right]
      \right] \\
    =& 
      \uf{n}\sum_{i=1}^n\left[
      \E_{\vec{S}\sim\measure^T}\left[
      \E_{(h.t)\sim\alg'(\vec{S})}\left[h(x_{t,i})\right]
      \right]   
      \right] 
    =
      \uf{n}\sum_{i=1}^n\left[
      \E_{\vec{S}\sim\measure^T}\left[
      \Pr_{(h.t)\sim\alg'(\vec{S})}\left[h(x_{t,i}) = 1\right]
      \right]    
      \right] \\ 
    =&
      \uf{n}\sum_{i=1}^n\left[
      \E_{\vec{S}\sim\measure^T}\left[
      \sum_{m=1}^T\Pr_{(h.t)\sim\alg'(\vec{S})}\left[
      h(x_{m,i}) = 1 \wedge t=m
      \right]
      \right]    
      \right] \\
    =& 
      \uf{n}\sum_{i=1}^n\Bigg[
      \E_{\vec{S}\sim\measure^T}\Bigg[
      \E_{\vec{z}\sim \vec{\measure}_i\mid \vec{S}} 
      \Bigg[ 
      \sum_{m=1}^T
      \Pr_{(h.t)\sim\alg'(\vec{S})}
      \left[ h(x_{m,i}) = 1 \wedge t=m \right]
      \Bigg]
      \Bigg]
      \Bigg] \numberthis \label{eq:5} ,
  \end{align*}
 where $\vec{z}=(z_1,\dots,z_T)$ is a vector s.t.\ $z_t\sim\measure_i(\cdot\mid S_t^{-i})$. 
 Given a multi-sample $\vec{S}$ and an element $z$, we write $\vec{S}^{(m,i)\leftarrow z}$ to denote the multi-sample $\vec{S}$ after replacing the $i^{th}$ element
  in the $m^{th}$ sample $S_m$ with $z$. Since $\alg'$ is $(\varepsilon,\delta)$-differentially private we get that the above is at most
  
  \begin{align*}
    \eqref{eq:5} 
    \leq& 
          \uf{n}\sum_{i=1}^n\Bigg[
          \E_{\vec{S}\sim\measure^T}\Bigg[
          \E_{\vec{z}\sim \vec{\measure}_i\mid \vec{S}} 
          \Bigg[
          \sum_{m=1}^T 
          e^\varepsilon 
          \Pr_{(h.t)\sim\alg'(\vec{S}^{(m,i)\leftarrow z_m})}
          \Bigg[ h(x_{m,i}) = 1 \wedge t=m \Bigg]
          + \delta
          \Bigg]    
          \Bigg]
          \Bigg] \\
    =& 
       T\delta + 
       e^\varepsilon\cdot
       \uf{n}\sum_{i=1}^n
       \sum_{m=1}^T\Bigg[
       \E_{\vec{S}\sim\measure^T}\Bigg[
       \E_{\vec{z}\sim \vec{\measure}_i\mid \vec{S}} 
         \Bigg[ \Pr_{(h.t)\sim\alg'(\vec{S}^{(m,i)\leftarrow z_m})}
          \Bigg[ h(x_{m,i}) = 1 \wedge t=m \Bigg]
          \Bigg]    
          \Bigg]
          \Bigg] \\
    =& 
       T\delta + 
       e^\varepsilon\cdot
       \uf{n}\sum_{i=1}^n
       \sum_{m=1}^T\Bigg[      
       \E_{\vec{S}\sim\measure^T}\Bigg[
       \E_{\vec{z}\sim \vec{\measure}_i\mid \vec{S}} 
         \Bigg[
         \Pr_{(h.t)\sim\alg'(\vec{S})}
          \left[ h(z_m) = 1 \wedge t=m \right]
          \Bigg]    
          \Bigg]
          \Bigg]  \\
    =& 
       T\delta + 
       e^\varepsilon\cdot
       \uf{n}\sum_{i=1}^n
       \sum_{m=1}^T\Bigg[      
       \E_{\vec{S}\sim\measure^T}\Bigg[
       \E_{\vec{z}\sim \vec{\measure}_i\mid \vec{S}} 
         \Bigg[
         \Pr_{(h.t)\sim\alg'(\vec{S})}
          \left[ h(z_t) = 1 \wedge t=m \right]
          \Bigg]    
          \Bigg]
          \Bigg]  \\
    =& 
       T\delta + 
       e^\varepsilon\cdot
       \uf{n}\sum_{i=1}^n
       \Bigg[      
       \E_{\vec{S}\sim\measure^T}\Bigg[
       \E_{\vec{z}\sim \vec{\measure}_i\mid \vec{S}} 
         \Bigg[
         \sum_{m=1}^T
         \Pr_{(h.t)\sim\alg'(\vec{S})}
          \left[ h(z_t) = 1 \wedge t=m \right]
          \Bigg]    
          \Bigg]
          \Bigg]  \\
    =& 
       T\delta + 
       e^\varepsilon\cdot
       \uf{n}\sum_{i=1}^n\Bigg[      
       \E_{\vec{S}\sim\measure^T}\Bigg[ 
       \E_{\vec{z}\sim \vec{\measure}_i\mid \vec{S}} 
         \left[ \Pr_{(h.t)\sim\alg'(\vec{S})}
          \left[ h(z_t) = 1 \right]
          \right]    
          \Bigg]
          \Bigg]  \\
    =& 
       T\delta + 
       e^\varepsilon
       \uf{n}\sum_{i=1}^n\left[      
       \E_{\vec{S}\sim\measure^T}\left[ \E_{\vec{z}\sim \vec{\measure}_i\mid \vec{S}}\left[
          \E_{(h.t)\sim\alg'(\vec{S})}
          \left[ h(z_t) \right]
          \right]    
          \right]
          \right]  \\
    =& 
       T\delta + 
       e^\varepsilon
       \uf{n}\sum_{i=1}^n\left[      
       \E_{\vec{S}\sim\measure^T}\left[
       \E_{(h.t)\sim\alg'(\vec{S})} 
         \left[
          \E_{\vec{z}\sim \vec{\measure}_i\mid \vec{S}}
          \left[      
          h(z_t)
          \right]
          \right]    
          \right]
          \right] \\
    =& 
       T\delta + 
       e^\varepsilon\cdot
       \uf{n}\sum_{i=1}^n
       \Bigg[   \E_{\vec{S}\sim\measure^T}
       \Bigg[   
       \E_{(h.t)\sim\alg'(\vec{S})} 
         \left[
          \E_{z\sim \measure_i(\cdot\mid S_t^{-i})}
          \left[      
          h(z)
          \right]
          \right]    
          \Bigg]
          \Bigg]. \numberthis \label{eq:6}
  \end{align*}
Since total variation is 
a special case of the Wasserstein metric
$\mathcal{W}_1$, 
Kantorovich-Rubinstein
duality implies that
for two probability measures $\mu,\nu$
  on a space $\X$
  and any function $h:\X\to[0,1]$,
  we have
$\abs{\E_{z\sim\mu}[h(x)] - \E_{z\sim\nu}[h(z)]}\leq \tv{\mu}{\nu}$. Applying this to
  $\measure_i(\cdot\mid S_t^{-i})$ and $\measure_i$
  we get that the above is at most
  \begin{align*}
    \eqref{eq:6}
    \leq&
      T\delta + 
      e^\varepsilon\cdot
      \uf{n}\sum_{i=1}^n\bigg[
      \E_{\vec{S}\sim\measure^T}\Big[ 
      \E_{(h.t)\sim\alg'(\vec{S})}
        \big[
      \E_{z\sim \measure_i}\big[
      h(z)
      \big]
      +
      \tv{\measure_i(\cdot\mid S_t^{-i})}{\measure_i}
      \big]    
      \Big]
      \bigg] \\
    \leq&
      \mixing +       
      T\delta + 
      e^\varepsilon\cdot
      \E_{\vec{S}\sim\measure^T}\Bigg[
      \E_{(h.t)\sim\alg'(\vec{S})} 
        \Big[
        \uf{n}\sum_{i=1}^n
      \E_{z\sim \measure_i}\big[      
      h(z)
      \big]
      \Big]
      \Bigg] \\
    =& 
      \mixing +       
      T\delta + 
      e^\varepsilon\cdot
      \E_{\vec{S},\alg'(\vec{S})}\left[ 
      h(\mu)
      \right] \\
    \leq&
      \mixing + T\delta +  e^\varepsilon - 1 + 
      \E_{\vec{S},\alg'(\vec{S})}\left[ 
      h(\mu)
      \right],       
  \end{align*}
  where the last inequality is due to the fact that
  $ye^\varepsilon \leq e^\varepsilon - 1 + y$
  for $y \leq 1$ and $\varepsilon \geq 0$. In summary,
  \begin{align*}
    \E_{\vec{S}\sim\measure^T} &
    \left[
      \E_{(h.t)\sim\alg'(\vec{S})}[h(S_t)]
    \right] 
    \leq
    \mixing + T\delta +  e^\varepsilon - 1 + 
    \E_{\vec{S},\alg'(\vec{S})}\left[ 
      h(\mu)
    \right].       
  \end{align*}

An identical argument yields
  \begin{align*}
    \E_{\vec{S}\sim\measure^T} &
    \left[
      \E_{(h.t)\sim\alg'(\vec{S})}[h(S_t)]
    \right] 
    \geq
    \mixing + T\delta +  e^\varepsilon - 1 + 
    \E_{\vec{S},\alg'(\vec{S})}\left[ 
      h(\mu)
    \right];       
  \end{align*}
  combining the two completes the proof.
\end{proof}

We use Lemma~\ref{lem:dp-expect-bound} to provide a high-probability generalization bound
for differentially private algorithms. Our main theorem in this setting (Theorem~\ref{thm:main-result-dp}) will be an immediate corollary of this bound.

\begin{theorem}[High probability bound]
  \label{thm:dp-high-prob-bound}
  Let $\varepsilon\in (0,1/3)$, $\delta \in (0,\varepsilon/4)$
  and $n \geq \frac{\log(2k\varepsilon/\delta)}{\varepsilon^2}$.
  Let $\alg:\X^n \to (2^\X)^k$ be an $(\varepsilon,\delta)$-differentially private algorithm.
  Let $\measure$ be a distribution over $\X^n$
  and
$S$ be a sample of size $n$
  drawn from $\measure$,
  and let $h_1,\ldots,h_k$ be the output of $\alg(S)$.
  Then
  \begin{equation*}
    \Pr_{S,\alg(S)}\left[\max_{i\in [k]}\abs{h_i(\measure)-h_i(S)}\geq 10\varepsilon+2\mixing\right]\leq \frac{\delta}{\varepsilon}      .
  \end{equation*}
\end{theorem}

The proof of Theorem~\ref{thm:dp-high-prob-bound} is almost identical to the analysis of \citet{bassily2016algorithmic}. It appears in the appendix for completeness. Intuitively, the proof is as follows.
We assume, towards contradiction, that there may be a differentially private algorithm that does not enjoy strong generalization guarantees.
We then use this mechanism to describe a different differentially private algorithm with a ``boosted  inability'' to generalize. That is, the proof goes by saying that if there is a differentially private algorithm whose generalization properties are not ``very good'' then there must exist a differentially private algorithm whose generalization properties are ``bad'', to the extent that contradicts Lemma~\ref{lem:dp-expect-bound}.

\medskip
Our main result (Theorem~\ref{thm:main-result-dp}) now follows as a corollary of  Theorem~\ref{thm:dp-high-prob-bound}.

\begin{proof}[Proof of Theorem~\ref{thm:main-result-dp}]
  $\mechanism$ is $(\varepsilon,\delta)$-differentially private.
  Since $\adversary$ can only access the data via $\mechanism$,
  we can treat the pair $\adversary,\mechanism$ as a single algorithm $\alg$, which gets a sample $S\sim\measure$ as input and returns $k$ predicates, as output.
  By closure to post-processing, $\alg$
  is also $(\varepsilon,\delta)$-differentially private.
  Applying Theorem~\ref{thm:dp-high-prob-bound} on $\alg$ we get that
  \[
    \Pr\left[
      \max_{i\in [k]}\abs{h_i(\measure)-h_i(S)}
      \geq 10\varepsilon+2\mixing
    \right]
    \leq \frac{\delta}{\varepsilon}.
  \]
  
  Since $\mechanism$ is $(\alpha,\beta)$-empirically-accurate
  it holds that
  \[
    \Pr\left[
      \max_{i\in [k]} \abs{q_i(S) - a_i} > \alpha
    \right] \leq \beta.
  \]
Combining these two bounds with the triangle inequality, we get
  \[
    \Pr\left[
      \max_{i\in [k]} \abs{q_i(\measure) - a_i}
      > \alpha + 10\varepsilon+2\mixing
    \right]
    < \beta + \frac{\delta}{\varepsilon}. \qedhere
  \]
\end{proof}

\subsection{A Tight Negative Result for Differential Privacy and Gibbs-Dependence}

In this section, we construct a distribution which is $\mixing$-Gibbs-Dependant,
and describe a differentially-private algorithm whose generalization gap w.r.t.\ this distribution is at least $\mixing$.
Hence, in a sense, the $\mixing$ factor attained on Theorem~\ref{thm:main-result-dp} is tight up to a constant.
Let $\X=[0,1]$ and define a measure $\measure$ over $\X^n$ by the following random process:
\begin{enumerate}
\item Sample a point $\seed\sim\uniform{[0,1]}$.
\item For every $i \in [n]:$ 
  \begin{enumerate}
  \item Sample $\sigma \sim \bernulli{\mixing}$.
    \begin{enumerate}
    \item If $\sigma = 1$ then $x_i = \seed$.
    \item Otherwise $x_i \sim \uniform{[0,1]}$
    \end{enumerate}
  \end{enumerate}
  
\item Return $\sample = \left(x_1,\ldots,x_n\right)$
\end{enumerate}

\begin{lemma}
  \label{lem:negative-1}
  The measure defined by the above process
  has $\mixing$-Gibbs-dependency.
\end{lemma}

\begin{proof}
  Initially, every marginal distribution is just uniform, i.e.
  $\measure_i \sim \uniform{[0,1]}$ and hence,
  for every $A\subseteq [0,1]$ it holds that
  $\measure_i(A) = |A|$.
  After conditioning, for every possible $x^{-i}$ and $x^*$, 
  we get that
  \begin{align*}
    \measure_i(A\mid x^{-i},x^*) 
    = \measure_i(A\setminus\{\seed\}\mid x^{-i},x^*) + \measure_i(A\cap\{\seed\}\mid x^{-i},x^*) 
    \in \Big( |A|(1-\mixing),\; |A|(1-\mixing)+\mixing\Big).      
  \end{align*}
  Since the above holds for every choice of $x^*$, we also have that
    \[
    \measure_i(A\mid x^{-i})
    \in \Big( |A|(1-\mixing),\; |A|(1-\mixing)+\mixing\Big).
  \]
  Therefore, for every $A\subseteq [0,1]$ it holds that
  \begin{align*}
    &\abs{\measure_i(A) - \measure_i(A\mid x^{-i})} 
    \leq 
    \max\left\{ |A|-|A|(1-\mixing) \;,\;  
    |A|(1-\mixing)+\mixing - |A|
    \right\} \leq \mixing.      
  \end{align*}
  So
  $\tv{\measure_i(\cdot)}{\measure_i(\cdot\mid x^{-i})} \leq \mixing$.
  Plunging this bound to the Gibbs-dependency definition yields
  \begin{align*}
    \mixing (\measure)
    =
    \sup_{x\in\X^n}
    \E_{i\sim[n]}
    \tv{\measure_i(\cdot)}{\measure_i(\cdot\mid x^{-i})}
    \leq
    \mixing.
  \end{align*}
\end{proof}

We next describe an algorithm that, despite being differentially private, performs ``badly'' when executed on samples from the above measure $\measure$. Specifically, this algorithms is capable of identifying a predicate with generalization error $\Omega(\psi)$. This shows that our connection between differential privacy and generalization (in the correlated setting) is tight, in the sense that the generalization error of differentially private algorithms {\em can} grow with $\psi$. This matches our positive result (see Theorem~\ref{thm:main-result-dp}).

Our algorithm is specified in Algorithm~\ref{alg:dp-negative}. As a subroutine, we use the following result of \citet{DBLP:journals/jmlr/BunNS19} for privately computing histograms.
\begin{theorem}[Private histograms, \citep{DBLP:journals/jmlr/BunNS19}]\label{thm:hist}
There exists an $(\varepsilon,\delta)$-differentially private algorithm that takes an input dataset $S\in\X^n$ and returns an a list $L\subseteq\X$ such that the following holds with probability at least $1-\beta$.
\begin{enumerate}
    \item For every $x\in\X$ that appears at least $\bigO{\frac{1}{\varepsilon}\log\frac{1}{\beta\delta}}$ times in $S$ we have that $x\in L$.
    \item For every $x\in L$ we have that $x$ appears at least {\em twice} in $S$.
\end{enumerate}
\end{theorem}

\begin{algorithm}
  \caption{Deviating Private Algorithm}
  \label{alg:dp-negative}
\begin{algorithmic}
  \STATE {\bfseries Input:} A sample $\sample$, privacy parameters $\varepsilon,\delta$. 
  \STATE {\bfseries Tool used:} An $(\varepsilon,\delta)$-DP algorithm $\histo$ for histograms.
  \STATE $L \leftarrow \histo(\sample,\varepsilon,\delta)$
  \IF{$L$ is empty}
  \STATE Return $h\equiv0$
  \ELSE
  \STATE Let $x$ be an arbitrary element in $L$
  \STATE Define $h:\X\to [0,1]$ as $h = \indicator{x}$
  \STATE Return $h$
  \ENDIF
\end{algorithmic}
\end{algorithm}

\begin{lemma}
  \label{lem:negative-2}
  For every 
  $\beta>0$,
  every $n \geq \bigO{\frac{1}{\mixing\varepsilon}\log\frac{1}{\beta\delta}}$,
  and for every $\mixing < 1$
  Algorithm~\ref{alg:dp-negative}
  is $(\varepsilon,\delta)$-differentially private
  and it outputs a predicate $h:\X\to [0,1]$ s.t.
  \[
    \Pr\left[\abs{h(S) - h(\measure)} \geq \frac{\mixing}{2}\right]
    > 1 - \beta - \exp\left(-\frac{n}{8} \right).
  \]
\end{lemma}

\begin{proof}
First observe that Algorithm~\ref{alg:dp-negative} is $(\varepsilon,\delta)$-differentially private, as it merely post-processes the outcome of the private histogram algorithm.

Next observe that, by the definition of the underlying measure $\measure$, and by our choice of $n$, w.h.p.,\ there are many copies of $x^*$ in the dataset $S$. Formally, by the Chernoff bound,
  \begin{align*}
    \label{eq:negative-1}
    \Pr&\left[\uf{n}\abs{\{x' \in \sample \mid x' = \seed\}} < \uf{2}\mixing\right] 
    =
    \Pr\left[\uf{n}\sum_{i=1}^{n}\sigma_i < \uf{2}\mixing\right]
    \leq
    \exp\left(-\frac{n}{8} \right).    
  \end{align*}
In addition, the probability of any element $x\neq x^*$ appearing more than once in $S$ is simply zero. Thus, with probability at least $1-\exp\left(-\frac{n}{8} \right)$ we have that $x^*$ appears in $S$ at least $n\psi/2=\Omega(\frac{1}{\varepsilon}\log(\frac{1}{\beta}{\delta}))$ times, and every other element appears in $S$ at most once. By the properties of the private histogram algorithm (see Theorem~\ref{thm:hist}), in such a case, with probability at least $1-\beta$ we have that $L=\{x^*\}$, and Algorithm~\ref{alg:dp-negative} returns the hypothesis $h = \indicator{\seed}$. As $x^*$ appears many times in $S$, this predicate has ``large'' empirical value. On the other hand,
for such predicate it holds that 
\begin{align*}
  h(\measure)&
  = \E_{\bar{\seed},\bar{x_1},\ldots,\bar{x_n}}
  \left[\uf{n}\left(\sum_{i=1}^{n}h(\bar{x_i})\right)\right] 
  = \Pr_{\bar{\seed},\bar{x_1},\ldots,\bar{x_n}}
  \left[\uf{n}\left(\sum_{i=1}^{n}\indicator{\bar{x_i}=\seed}\right)\right]
  = 0    
\end{align*}
as the probability that for a fresh new sampling we will get $\bar{\seed}=\seed$ is zero,
implying that the probability that any point in the sample to be $\seed$ is also zero.

Overall, with probability at least $1-\beta-\exp\left(-\frac{n}{8} \right)$, the algorithm returns a predicate $h$ such that $h(S)\geq\psi/2$ but $h(\measure)=0$. 
  
\end{proof}

\subsection{Application to Markov Chains}
\label{sec:umc}
In this section, we demonstrate an application of our tools and results regarding Gibbs-dependency and differential privacy to the problem of learning \emph{Markov chains} adaptively.
For our notion of dependence, it will be more convenient to analyze the 
\emph{Undirected Markov Chains}.
By the Hammersley-Clifford theorem
\citep{HammersleyClifford:1971,Clifford90}, every Markov measure on a chain graph with nonzero transition probabilities can be factorized according to pairwise potential functions (formalized below),
which we refer to as the {\em undirected Markov chain} formalization \citep{kontorovich12}.

The formal definition of an undirected Markov chain measure is as follows.

\begin{definition}
A measure $\measure$ over $\Omega^n$ is an {\em undirected Markov chain} if there are positive
functions $\{\potential_i\}_{i\in [n-1]}$, called {\em potential functions}, such that for any $x\in\Omega^n$
\[
\measure(x) = \frac{\prod_{i=1}^{n-1}\potential_i(x_i,x_{i+1})}
{\sum_{x'\in\Omega^n}\prod_{i=1}^{n-1}\potential_i(x'_i,x'_{i+1})}.
\]
\end{definition}
This is a special case of the more general \emph{undirected graphical model} (see \citet{lauritzen1996graphical}).
For the sake of convenience, we will use the following notations.

\begin{definition}
Let $\measure$ be an undirected Markov chain with potential functions $\{g_i\}_{i\in[n-1]}$. We denote the maximal and minimal potentials as follows.
\begin{itemize}
    \item $R_i(\measure) = \max_{a,b\in \Omega}\potential_i(a,b)$,
    \item $r_i(\measure) = \min_{a,b\in \Omega}\potential_i(a,b)$,
    \item $R(\measure) = \max_i \{R_i(\measure)\}$,
    \item $r(\measure) = \min_i \{r_i(\measure)\}$,
    \item $\bar{R}(\measure) := \frac{R(\measure)^2-r(\measure)^2}{R(\measure)^2+r(\measure)^2}$.
\end{itemize}
When $\measure$ is clear from the context, we simply write $R_i,r_i,R,r,\bar{R}$ instead of $R_i(\measure),r_i(\measure),R(\measure),r(\measure),\bar{R}(\measure)$.
\end{definition}

In order to apply our techniques to the case where the underlying distribution is an undirected Markov chain, we need to bound the Gibbs-dependency of undirected Markov chains. We first show the following lemma. (The proof of this lemma is deferred to a later part of this section.)

\begin{lemma}\label{lem:markovbound}
For every undirected Markov chain $\measure$ we have
\[
\mixing(\measure) \leq \bar{R}:=
\frac{R^2 - r^2}{R^2 + r^2}.
\]
\end{lemma}

That is, the above lemma bounds the Gibbs-dependency of undirected Markov chains as a function of the potential functions. Combining this bound with Corollary~\ref{cor:main}, we obtain the following result.

\begin{corollary}\label{cor:barR}
There exists a computationally efficient mechanism for answering $k$ adaptively chosen queries with the following properties. When given $n\geq 
m
=
\tilde{O}\left( \frac{\sqrt{k}}{\alpha^2}\log\frac{1}{\beta} \right)$ samples (an $n$-tuple) from an (unknown) undirected Markov chain $\measure$, the mechanism guarantees $\left(\alpha+2\bar{R}(\measure),\beta\right)$-statistical-accuracy (w.r.t.\ the underlying distribution $\measure$).
\end{corollary}

In particular, Corollary~\ref{cor:barR} shows that if the underlying chain $\measure$ satisfies $\bar{R}(\measure)\leq\alpha$, then the dependencies in $\measure$ can be ``accommodated for free'', in the sense that we can efficiently answer the same amount of adaptive queries as if the underlying distribution is a product distribution. We are not aware of an alternative method for answering this amount of adaptive queries under these conditions. As we next explain, we can broaden the  applicability of our techniques even further, by reducing  dependencies in the data as follows. The idea is to access only a part of the chain, obtained by ``skipping'' a fixed number of elements between two random samples. Formally,

\begin{definition}[Skipping Samples]\label{def:skipping}
Given a measure $\measure$ over $n$-tuples, and an integer $t$, we define the measure $\measure_{\times t}$ over $\frac{n}{t}$-tuples as follows.\footnote{We assume here for simplicity that $t$ divides $n$.} To sample from $\measure_{\times t}$, let $(x_0,x_1,x_2,x_3,\dots,x_{n-1})\sim\measure$, and return $(x_0,x_{t},x_{2t},x_{3t},\dots,x_{n-t})$.
\end{definition}

Intuitively, as Markov chains are ``memoryless processes'', skipping points in our sample (as in Definition~\ref{def:skipping}), should significantly reduce dependencies within the remaining points. We formalize this intuition and prove the following theorem. (The proof of this theorem is deferred to a later part of this section.)

\begin{theorem}
  \label{thm:gibbs_chains}
  For every undirected Markov chain $\measure$ and for every $t$ we have
  \[\mixing(\measure_{\times t}) \leq \mixing(\measure)^t.\]
\end{theorem}

That is, Theorem~\ref{thm:gibbs_chains} states that by reducing our sample size {\em linearly} with $t$, we could reduce dependencies within our sample {\em exponentially} in $t$. Combining this bound with Corollary~\ref{cor:main}, we obtain the following result.

\begin{corollary}
There exists a computationally efficient mechanism that is $(3\alpha,\beta)$-statistically-accurate for $k$ adaptively chosen queries, given a sample (an $n$-tuple) drawn from an underlying distribution $\measure_{\times t}$, where $\measure$ is an undirected Markov-chain, and where
$$
n\geq \tildeO{\frac{\log(1/\beta)\sqrt{k}}{\alpha^2}}\qquad\text{and}\qquad
t\geq \frac{\log (1/\alpha)}{\log (1/\bar{R})}.
$$
\end{corollary}

\begin{remark}
As a baseline, one can choose the ``skipping parameter'' $t$ to be sufficiently big s.t.\ the Gibbs-dependency would drop below $\beta/n$. 
As we mentioned in Section~\ref{subsec:into-gd-dp},
in that case the dependencies in the data would be small enough to the extent we could simply apply existing tools for answering queries {\em w.r.t.\ product distributions}, in order to answer adaptive queries w.r.t.\ $\measure_{\times t}$. However, this would require the skipping parameter $t$ to be as big as 
$\frac{\log(n/\beta)}{\log(1/\bar{R})}$, i.e., to increase by (roughly) a $\log(n)$ factor, which in turn, would result in a larger sample complexity.
\end{remark}

We next prove Lemma~\ref{lem:markovbound} and Theorem~\ref{thm:gibbs_chains}.

\begin{proof}[Proof of Lemma~\ref{lem:markovbound}]
For any $i\in [2,n-1]$,
\footnote{The case of $i\in\set{1,n}$ 
has an almost identical argument;
only the $g_{i-1}(v_{i-1},a)$ (respectively, $g_n(a,v_{i+1})$) factor
is omitted.
This does not affect the rest of the argument for the upper bound.
}
$a \in \Omega$ and $u,v \in \Omega^n$
\begin{align*}
  \measure_i(a\mid v^{-i})
  = \frac{\potential_{i-1}(v_{i-1},a)\potential_{i}(a,v_{i+1})}
  {\sum_{a'}\potential_{i-1}(v_{i-1},a')\potential_{i}(a',v_{i+1})}
\end{align*}

We will be using the following lemma of \citet{kontorovich12}:
\begin{lemma}
  For $n\in\N$ and $0\leq r\leq R$, consider the vectors ${\alpha}\in [0,\infty)^n$ and $f,g\in [r,R]^n$.
  Then,
  \[
  \uf{2}\sum_{i=1}^n \abs{\frac{\alpha_i f_i}{\sum_{j=1}^n \alpha_j f_j} - \frac{\alpha_i g_i}{\sum_{j=1}^n \alpha_j g_j}}
  \leq
  \frac{R-r}{R+r}
  .
  \]
\end{lemma}
We apply the lemma using
\begin{itemize}
\item $f_a = \potential_{i-1}(v_{i-1},a)\potential_{i}(a,v_{i+1})$
\item $h_a = \potential_{i-1}(u_{i-1},a)\potential_{i}(a,u_{i+1})$
\item $\alpha_a = 1$
\end{itemize}

and get that

\begin{align*}
  \uf{2}\sum_{a}\abs{\measure_i(a\mid u^{-i}) - \measure_i(a\mid v^{-i})}
  \leq \frac{R_{i-1}R_i - r_{i-1}r_i}{R_{i-1}R_i + r_{i-1}r_i}.
\end{align*}
It follows that
\begin{align*}
  \tvleft \measure_i&(\cdot)- \measure_i(\cdot\mid v^{-i})\tvright 
  = \uf{2}\sum_{a}\abs{\measure_i(a) - \measure_i(a\mid v^{-i})} \\
  =& \uf{2}\sum_{a}\abs{
    \sum_{u^{-i}}\measure_i(a\mid u^{-i})\measure^{-i}(u^{-i}) - \measure_i(a\mid v^{-i})} 
  = \uf{2}\sum_{a}\Bigg|\sum_{u^{-i}}\measure_i(a\mid u^{-i})\measure^{-i}(u^{-i}) 
    - \sum_{u^{-i}}\measure^{-i}(u^{-i})\measure_i(a\mid v^{-i})\Bigg| \\
  =& \uf{2}\sum_{a}\abs{\sum_{u^{-i}}\measure^{-i}(u^{-i})\left[
    \measure_i(a\mid u^{-i}) - \measure_i(a\mid v^{-i})
    \right]} 
  \leq \uf{2}\sum_{a}\sum_{u^{-i}}\measure^{-i}(u^{-i})\abs{
    \measure_i(a\mid u^{-i}) - \measure_i(a\mid v^{-i})} \\
  =& \sum_{u^{-i}}\measure^{-i}(u^{-i})\uf{2}\sum_{a}\abs{
    \measure_i(a\mid u^{-i}) - \measure_i(a\mid v^{-i})} 
  \leq \sum_{u^{-i}}\measure^{-i}(u^{-i})\frac{R_{i-1}R_i - r_{i-1}r_i}{R_{i-1}R_i + r_{i-1}r_i}
    = \frac{R_{i-1}R_i - r_{i-1}r_i}{R_{i-1}R_i + r_{i-1}r_i}
    .
\end{align*}
Finally,
\begin{align*}
  \mixing(\measure)
  = \sup_{v}\E_{i}\tv{\measure_i(\cdot)}{\measure_i(\cdot\mid v^{-i})} \leq \frac{R^2 - r^2}{R^2 + r^2}
  .
\end{align*}
\end{proof}

In order to prove Theorem~\ref{thm:gibbs_chains}, we first
establish the following notations:
\begin{itemize}
\item $x_{i \pm t} = x_{i-1},x_{i+t}$

\item $c_i^t := \sup\limits_{x_{i\pm t},y_{i\pm t}} \tvleft\measure_{i\pm (t-1)}(\cdot\mid x_{i\pm t}) - \measure_{i\pm(t-1)}(\cdot\mid y_{i\pm t})\tvright$
\item $\symmix{i}{t}
  := \sup\limits_{x_{i\pm t},y_{i\pm t}}\tv{\measure_i(\cdot\mid x_{i\pm t})}{\measure_i(\cdot\mid y_{i\pm t})}$
\end{itemize}
Note that $$c_i^1 = \symmix{i}{i+1}
= \sup_{x_{i\pm 1},y_{i\pm 1}}\tv{\measure_i(\cdot\mid x_{i\pm 1})}{\measure_i(\cdot\mid y_{i\pm 1})}.$$

We will be using the following two lemmas (we prove these two lemmas after the proof of Theorem~\ref{thm:gibbs_chains}).

\begin{lemma}
  \label{lem:chains}
  $\symmix{i}{t} \leq \prod_{j=1}^{t}c_i^j$
\end{lemma}

\begin{lemma}
  \label{lem:chains_technical}
  For every $t$ and every $i$, there exist some $j$ s.t.
  $c_i^t\leq c_j^1$.
\end{lemma}

We now prove Theorem~\ref{thm:gibbs_chains} using Lemmas~\ref{lem:chains} and~\ref{lem:chains_technical}.

\begin{proof}[Proof of Theorem~\ref{thm:gibbs_chains}]
  Combining Lemma~\ref{lem:chains} and Lemma~\ref{lem:chains_technical} yields that
  for every undirected Markov measure $\measure$ and for every $t$,
  \begin{align*}
  \label{eq:chains1}
    \max_{i}&\symmix{i}{t} 
    \leq \max_{i}\prod_{j=1}^{t}c_i^j
    \leq \max_{i}\prod_{j=1}^{t}c_{l(j)}^1 
    \leq \max_{i}\max_{l}(c_{l}^1)^t 
    = (\max_{i} c_{l}^1)^t 
    = (\mixing(\measure))^t, \numberthis   
  \end{align*}
  where the first inequality is due to Lemma~\ref{lem:chains}, the second is by Lemma~\ref{lem:chains_technical} \footnote{The function $l:[n]\to [n]$ returns for every coordinate $i$ the appropriate coordinate $l(i)$ which is guaranteed by Lemma~\ref{lem:chains_technical} to bound it.}. 
  The Last equality holds by the definitions of $\mixing$ and $c_l^1$.
  Since
  \begin{align*}
    \measure_i(\cdot) = 
    \sum_{x_{i\pm t}\in\Omega^2}
    \measure_i(\cdot\mid x_{i\pm t})
    \measure_{i\pm t}(x_{i\pm t})
    ,
  \end{align*}
  we have, by the undirected Markov property,
  \begin{align*}
  \mixing(\measure_{\times t})
      &=\max_i\sup_{y_{i\pm t}}\tv{\measure_i(\cdot)}{\measure_i(\cdot\mid y_{i\pm t})} \\
      &=
      \max_i\sup_{y_{i\pm t}}\tvleft
      \sum_{x_{i\pm t}\in\Omega^2}
      (\measure_i(\cdot\mid x_{i\pm t}) - \measure_i(\cdot\mid y_{i\pm t}))
      \measure(x_{i\pm t})\tvright\\
      &\leq 
      \max_i\sup_{y_{i\pm t}}\sum_{x_{i\pm t}\in\Omega^2} \tv{\measure_i(\cdot\mid x_{i\pm t})}{\measure_i(\cdot\mid y_{i\pm t})}\measure(x_{i\pm t}) \\    
      &\leq
      \max_{i} \symmix{i}{t} \leq (\mixing(\measure))^t,
  \end{align*}
  where the last inequality is due to \eqref{eq:chains1}.
\end{proof}

\begin{proof}[Proof of Lemma~\ref{lem:chains}]
  Let $x_{i\pm t},y_{i\pm t}$ be some pairs of realization for the $i-t,i+t$ variable in the chain. By the law of total probability,
  \begin{align*}
    &\tv{\measure_i(\cdot\mid x_{i\pm t})}{\measure_i(\cdot\mid y_{i\pm t})} \\
    &=
\tvleft
    \sum_{x_{i\pm (t-1)}} \measure_i(\cdot\mid x_{i\pm (t-1)})\measure_{i\pm (t-1)}(x_{i\pm (t-1)}\mid x_{i\pm t}) 
    - \sum_{y_{i\pm (t-1)}} \measure_i(\cdot\mid y_{i\pm (t-1)})\measure_{i\pm (t-1)}(y_{i\pm (t-1)}\mid y_{i\pm t})\tvright. \numberthis \label{eq:umc-total}
  \end{align*}
  Define a coupling measure $\Pi_{i\pm (t-1)}(\cdot,\cdot\mid x_{i\pm t},y_{i\pm t})$
  whose marginals
  are $\mu_{i\pm (t-1)}(\cdot\mid x_{i\pm t})$ and $\mu_{i\pm (t-1)}(\cdot\mid y_{i\pm t})$.
  Then
  \begin{align*}
    \eqref{eq:umc-total}&= \tvleft\sum_{x_{i\pm (t-1)}}\sum_{y_{i\pm (t-1)}}    
    (\measure_i(\cdot\mid x_{i\pm (t-1)}) - \measure_i(\cdot\mid y_{i\pm (t-1)}))
    \Pi_{i\pm (t-1)}(x_{i\pm (t-1)},y_{i\pm (t-1)}\mid x_{i\pm t},y_{i\pm t})\tvright\\
    &\leq
    \sum_{x_{i\pm (t-1)}}\sum_{y_{i\pm (t-1)}}   
    \tv{\measure_i(\cdot\mid x_{i\pm (t-1)})}{\measure_i(\cdot\mid y_{i\pm (t-1)})} 
    \Pi_{i\pm (t-1)}(x_{i\pm (t-1)},y_{i\pm (t-1)}\mid x_{i\pm t},y_{i\pm t})\\
    &\leq
    \symmix{i}{t-1}
    \sum_{x_{i\pm (t-1)}}\sum_{y_{i\pm (t-1)}}   
    1_{x_{i\pm (t-1)} \neq y_{i\pm (t-1)}}
    \Pi_{i\pm (t-1)}(x_{i\pm (t-1)},y_{i\pm (t-1)}\mid x_{i\pm t},y_{i\pm t}).
  \end{align*}
  By the dual form of the total variation distance,\footnote{
  By the Kantorovich-Rubinstein duality of the specific case of total-Variation distance 
  $
    \tv{P}{Q}
    =
    \min_{\Pi\in\Delta(P,Q)}\int_{\Omega}\int_{\Omega}
    1_{x\ne y}
    d\Pi(x,y)
  $
  when $\Delta(P,Q)$ is the set of all the possible coupling of $P$ and $Q$.}
we can choose $\Pi_{i\pm (t-1)}$ to be such that
\begin{align*}
  &\tv{\measure_{i\pm (t-1)}(\cdot\mid x_{i\pm t})}{\measure_{i\pm (t-1)}(\cdot\mid y_{i\pm t})} \\
  &= 
  \sum_{x_{i\pm (t-1)}}\sum_{y_{i\pm (t-1)}}   
  1_{x_{i\pm (t-1)} \neq y_{i\pm (t-1)}}
  \quad\Pi_{i\pm (t-1)}(x_{i\pm (t-1)},y_{i\pm (t-1)}\mid x_{i\pm t},y_{i\pm t})  
\end{align*}
and therefore
\begin{align*}
  \tv{\measure_i(\cdot\mid x_{i\pm t})}{\measure_i(\cdot\mid y_{i\pm t})} 
  \leq
  \symmix{i}{t-1}
  \tv{\measure_{i\pm (t-1)}(\cdot\mid x_{i\pm t})}{\measure_{i\pm (t-1)}(\cdot\mid y_{i\pm t})}
  \leq
  \symmix{i}{t-1} c_i^t.  
\end{align*}
Hence we get that
\[
  \symmix{i}{t}
  = \sup_{x_{i\pm t},y_{i\pm t}}\tv{\measure_i(\cdot\mid x_{i\pm t})}{\measure_i(\cdot\mid y_{i\pm t})}
  \leq
  \symmix{i}{t-1} c_i^t
\]
and by induction we get the lemma's result.
\end{proof}

\begin{proof}[Proof of Lemma~\ref{lem:chains_technical}]
First we will show that for any $j,k$ the following holds
\begin{align*}
    \sup 
    &\tv{\measure_{j}(\cdot\mid x_{j-1},x_{j+k})}
    {\measure_{j}(\cdot\mid y_{j-1},y_{j+k})} 
    \leq
    \sup 
    \tv{\measure_{j}(\cdot\mid x_{j-1},x_{j+k-1})}
    {\measure_{j}(\cdot\mid y_{j-1},y_{j+k-1})}. \numberthis \label{eq:one-step}
\end{align*}
Indeed,
\begin{align*}
    &\sup 
    \tv{\measure_{j}(\cdot\mid x_{j-1},x_{j+k})}
    {\measure_{j}(\cdot\mid y_{j-1},y_{j+k})} 
    =\\
    &\sup\tvleft\sum_{x_{j+k-1}} \measure_j(\cdot\mid x_{j-1},x_{j+k-1})\measure_{j+k-1}(x_{j+k-1}\mid x_{j+k}) 
    - \sum_{y_{j+k-1}} \measure_j(\cdot\mid y_{j-1},y_{j+k-1})\measure_{j+k-1}(y_{j+k-1}\mid y_{j+k})\tvright.
\end{align*}
Let $\Pi_{j+k-1}(\cdot,\cdot\mid x_{j+k},y_{j+k})$ be a coupling distribution whose marginal distributions are
$\measure_{j+k-1}(y_{j+k-1}\mid y_{j+k})$
and
$\measure_{j+k-1}(x_{j+k-1}\mid x_{j+k})$,
we get that the above is equal to
\begin{align*}
    \sup\tvleft\sum_{x_{j+k-1}}&\sum_{y_{j+k-1}} (\measure_j(\cdot\mid x_{j-1},x_{j+k-1}) 
    - \measure_j(\cdot\mid y_{j-1},y_{j+k-1}))
    \Pi_{j+k-1}(x_{j+k-1},y_{j+k-1}\mid x_{j+k},y_{j+k})\tvright \\
    \leq
    \sup\sum&_{x_{j+k-1}}\sum_{y_{j+k-1}} \tvleft\measure_j(\cdot\mid x_{j-1},x_{j+k-1}) 
    - \measure_j(\cdot\mid y_{j-1},y_{j+k-1})\tvright
    \Pi_{j+k-1}(x_{j+k-1},y_{j+k-1}\mid x_{j+k},y_{j+k}) \\
    \leq
    \sup
    \tvleft\measure_j&(\cdot\mid x_{j-1},x_{j+k-1}) 
    - \measure_j(\cdot\mid y_{j-1},y_{j+k-1})\tvright.
\end{align*}
Now we turn to the quantity of interest:
\begin{align*}
\sup_{x_{i\pm t},y_{i\pm t}} 
&\tv{\measure_{i\pm (t-1)}(\cdot\mid x_{i\pm t})}
{\measure_{i\pm (t-1)}(\cdot\mid y_{i\pm t})} \\
=
\sup_{x_{i\pm t},y_{i\pm t}}& 
\tvleft\sum_{x_{i+t-1}}\measure_{i\pm (t-1)}(\cdot\mid x_{i\pm t},x_{i+t-1}) 
\measure_{i+t-1}(x_{i+t-1}\mid x_{i\pm t}) \\
&- \sum_{y_{i+t-1}}\measure_{i\pm (t-1)}(\cdot\mid y_{i\pm t},y_{i+t-1})
\measure_{i+t-1}(y_{i+t-1}\mid y_{i\pm t}) \tvright \\
=
\sup_{x_{i\pm t},y_{i\pm t}} 
&\tvleft\sum_{x_{i+t-1}}\measure_{i-t+1}(\cdot\mid x_{i- t},x_{i+t-1})
\measure_{i+t-1}(x_{i+t-1}\mid x_{i\pm t}) \\
&- \sum_{y_{i+t-1}}\measure_{i-t+1}(\cdot\mid y_{i- t},y_{i+t-1})
\measure_{i+t-1}(y_{i+t-1}\mid y_{i\pm t}) \tvright
.
\numberthis \label{eq:before}
\end{align*}
Let $\Pi_{i+t-1}(\cdot,\cdot\mid x_{i\pm t}),y_{i\pm t})$ be a coupling distribution whose marginals are $\measure_{i+t-1}(x_{i+t-1}\mid x_{i\pm t})$ and $\measure_{i+t-1}(y_{i+t-1}\mid y_{i\pm t})$.
Then the above is then equal to
\begin{align*}
    \eqref{eq:before}
    = \sup&_{x_{i\pm t},y_{i\pm t}} 
    \tvleft\sum_{x_{i+t-1}}\sum_{y_{i+t-1}}\measure_{i-t+1}(\cdot\mid x_{i-t},x_{i+t-1})
    - \measure_{i-t+1}(\cdot\mid y_{i-t},y_{i+t-1})\\
    &\Pi_{i+t-1}(x_{i+t-1},y_{i+t-1}\mid x_{i\pm t},y_{i\pm t})\tvright \\
    \leq
    \sup&_{x_{i\pm t},y_{i\pm t}} 
    \sum_{x_{i+t-1}}\sum_{y_{i+t-1}}
    \tvleft\measure_{i-t+1}(\cdot\mid x_{i-t},x_{i+t-1})
    - \measure_{i-t+1}(\cdot\mid y_{i-t},y_{i+t-1})\tvright\\
    &\Pi_{i+t-1}(x_{i+t-1},y_{i+t-1}\mid x_{i\pm t},y_{i\pm t})  .
\end{align*}
Plugging $j=i-t+1$ and $k=t-2$
into \eqref{eq:one-step} yields
\begin{align*}
    \sup_{x_{i\pm t},y_{i\pm t}} 
    &\tv{\measure_{i\pm (t-1)}(\cdot\mid x_{i\pm t})}
    {\measure_{i\pm (t-1)}(\cdot\mid y_{i\pm t})} \\
    \leq&
    \sup_{x_{i-t,i-t+2},y_{i-t,i-t+2}} 
    \tvleft\measure_{i-t+1}(\cdot\mid x_{i-t,i-t+2})
    - 
    \measure_{i-t+1}(\cdot\mid y_{i-t,i-t+2})\tvright
\end{align*}
which completes the proof.
\end{proof}

\section{Adaptive Learning Via Transcript Compression}
\label{sec:compression}

In this section we show how the notion of \emph{transcript compressibility} can be used to derive generalization bounds %
even if the data is not i.i.d.\ distributed. We start by recalling the notion of {\em transcript compression} by \citet{dwork2015generalization}. 
We denote by $\game_{n,k}(\analyst,\mechanism,\sample)$ the \emph{transcript} of the interaction between the mechanism $\mechanism$ and the analysis $\analyst$ during the adaptive accuracy game defined in Algorithm~\ref{alg:adapt-pop-game} with sample of size $n$ and $k$ queries. 
\begin{definition}[Transcript Compression \citep{dwork2015generalization}]\label{def:compression}
  We say that a mechanism $\mechanism$  enables
  \emph{transcript compression} to $b(n,k)$-bits,
  if for every deterministic analyst $\analyst$ there exist
  a set of possible transcripts $\mathcal{H}_\analyst$,
  of size $\abs{\mathcal{H}_\analyst} \leq 2^{b(n,k)}$,
  s.t. for every sample $\sample$ it holds that 
  $\Pr\left[ \game_{n,k}(\analyst,\mechanism,\sample)\in \mathcal{H} \right] = 1$.
\end{definition}
 
 Following \citet{DBLP:journals/corr/BassilyF16}, in this section we aim to design mechanisms that answer adaptively chosen queries while providing statistical accuracy, under the assumption that the given queries are {\em concentrated} around their expected value. Unlike \citet{DBLP:journals/corr/BassilyF16}, we aim to achieve this goal using the notion of {\em transcript compression}, rather than {\em typical-stability}. As we show, this allows for a significantly simpler analysis (and definitions). Formally,
 
 \begin{definition}\label{def:concentrated}
 Given a measure $\measure$ over $\X$, a query  $q:\X^n\rightarrow\R$, 
 and a parameter $\delta\in[0,1]$, we write $\gamma(q,\measure,\delta)$ to denote the minimal number $\gamma\in[0,1]$ such that
 \[
    \Pr_{S\sim\measure}\left[
      \abs{q(S) - \E_{T\sim \measure}[q(T)]} > \gamma
    \right] < \delta.
  \]
 \end{definition}
 
 That is, $\gamma(q,\measure,\delta)$ denotes the minimal number such that, without adaptivity, $q(S)$ deviates from its expectation by more than $\gamma(q,\measure,\delta)$  with probability at most $\delta$ when sampling $S\sim\measure$. 
 
 \begin{remark}
 The results in this section are not restricted to statistical queries. The results in this section hold for arbitrary queries (mapping $n$-tuples to the reals). 
 \end{remark}
 
 Consider again Algorithm~\ref{alg:adapt-pop-game}, and Definition~\ref{def:adaptiveaccuracy} (the definition of statistical accuracy). We now use Definition~\ref{def:concentrated} in order to introduce a relaxation for statistical accuracy, in which the mechanism is allowed to incur $\gamma(q,\measure,\delta)$ as an additional error.  
 
 \begin{definition}
  A mechanism $\mechanism$ is {\em $(\alpha,\beta,\delta)$-statistically-query-accurate} for $k$ rounds given $n$ samples, 
  if for every distribution $\measure$ over $n$-tuples, and every adversary $\adversary$, 
  it holds that 
  \[
    \Pr_{\substack{S\sim\measure\\\texttt{Game}(\mechanism,k,\adversary,S)}}\left[
      \max_{i\in [k]} \abs{q_i(\measure) - a_i} > \alpha+\gamma(q,\measure,\delta)
    \right] \leq \beta.
  \]
\end{definition}

\begin{remark}
  For a statistical query $q$ and a product measure  $\measure$, by Hoeffding's inequality, we get that 
  $\gamma(q,\measure,\delta) = \sqrt{\frac{1}{2n}\ln\frac{2}{\delta}}$.
  Hence, for the i.i.d.\ regime, for large enough samples, the definition of  $(\alpha,\beta,\delta)$-statistical-query-accuracy is in fact equivalent (up to factor 2) to the original definition of $(\alpha,\beta)$-statistical-accuracy (Definition~\ref{def:adaptiveaccuracy}).
\end{remark}

We observe that the analysis of \citet{dwork2015generalization} for transcript compression easily extends to non-i.i.d.\ measures when given concentrated queries. Somewhat surprisingly, this simple technique essentially matches the bounds obtained using typical stability \citep{DBLP:journals/corr/BassilyF16}.  In the next lemma we show that (w.h.p.)\ an analyst interacting with a transcript-compressing mechanism cannot identify a query that overfits to the date.
 
\begin{lemma}
  \label{lem:compression-gen}
  Let $\mechanism$ be a mechanism which enables
  transcript compression to $b(n,k)$-bits.
  For every measure $\measure$ and every analyst $\analyst$,
  $$
    \Pr_{S, \game_{n,k}}\left[ \exists i: \abs{q_i(S) - q_i(\measure)} \geq \gamma(q,\measure,\delta) \right]
    \leq \delta \cdot k \cdot 2^{b(n,k)}
  $$
\end{lemma}

\begin{proof}%
  Fix an analyst $\analyst$.
  By Definition~\ref{def:compression}, there exist a set of transcripts $H_\analyst$
  of size at most $2^{b(n,k)}$.
  As every transcript consists of at most $k$ queries,
  there can be at most $k2^{b(n,k)}$ possible queries over all possible interactions between $\analyst$ and $\mechanism$.
  Denote this set of possible queries as $Q_\analyst$. By a union bound we get that
    \[
    \Pr_{S \sim \measure}\left[ \bigvee_{q \in Q_{\analyst}} \abs{q_i(S) - q_i(\measure)} \geq \gamma(q,\measure,\delta) \right]
    \leq \delta \cdot  k \cdot 2^{b(n,k)},    
  \]
  and hence
  \[
    \Pr_{S, \game_{n,k}}\left[ \exists i: \abs{q_i(S) - q_i(\measure)} \geq \gamma(q,\measure,\delta) \right]
    \leq k \cdot \delta \cdot 2^{b(n,k)}.    
  \]
\end{proof}

Using the above lemma, we prove our main theorem for this section.

\begin{theorem}
  \label{thm:compression-accuracy}
  Let $\mechanism$ be a mechanism which enables transcript compression to $b(n,k)$ bits
  and also exhibits $(\alpha, \beta)$-empirical-accuracy for k rounds given n samples.
  Then $\mechanism$ is also 
  $(\alpha,\beta +\delta  k 2^{b(n,k)},\delta)$-statistically-query-accurate, for every choice of $\delta$.
\end{theorem}

\begin{proof}%
  As $\mechanism$ is $(\alpha, \beta)$-empirically-accurate and also
  enables transcript compression to $b(n,k)$ bits,
  by Lemma~\ref{lem:compression-gen} and the union bound 
  \begin{align*}
    \Pr_{S, \game_{n,k}}
    \big[
      \left(\exists i: \abs{q_i(S) - q_i(\measure)} > \gamma(q,\measure,\delta)  \right) 
      \vee
      \left(\exists i: \abs{q_i(S) - a_i} > \alpha \right)      
    \big] 
    \leq
    \beta + \delta \cdot k \cdot 2^{b(n,k)}.
  \end{align*}
    Hence by the triangle inequality
    \begin{align*}
      \Pr&_{S, \game_{n,k}}
      \left[ \exists i: \abs{a_i - q_i(\measure)} \geq \alpha + \gamma(q,\measure,\delta)  \right]
      \leq 
      \beta + \delta \cdot k \cdot 2^{b(n,k)}.
    \end{align*}
  \end{proof}

Applying Theorem~\ref{thm:compression-accuracy} together with the transcript-compressing mechanisms of \citet{dwork2015generalization}, we get the following two results.

\begin{theorem}
\label{thm:main-compression-efficient}
    For every $\alpha,\delta$, there exists an $(\alpha,\beta,\delta)$-statistically-query-accurate mechanism for $k$ rounds given $n$ samples, where
    $\beta = k \cdot \delta \cdot 2^{k\cdot\log\frac{1}{\alpha}}$. The mechanism is computationally efficient.
  \end{theorem}

\begin{theorem}
\label{thm:main-compression-inefficient}
    For every $\delta$, there exists an $(\alpha,\beta,\delta)$-statistically-query-accurate mechanism for $k$ rounds given $n$ samples, where
    $\alpha = \bigO{\left(\frac{\ln k}{n}\right)^{1/4}}$
    and 
    $\beta = k \cdot \delta \cdot 2^{\tildeO{\sqrt{n}\cdot\log|\X|\cdot(\log k)^{3/2}}}$. The mechanism is computationally inefficient.
  \end{theorem}

\bibliographystyle{apalike}

\appendix 

\section{Missing proofs}

\subsection{Product measure}
\label{apn:prod}
We show the following claim
\begin{claim}
For a measure $\measure\sim\X^n$,
if for every $i \in [n]$ and for every possible $x \in \X^n$ it holds that
$\measure_i = \measure_i(\cdot\mid x^{-i}$, then
$\measure$ is a product measure.
\end{claim}
\begin{proof}
  For convenience, we denote for every $i\leq j$ the following notation $a_{i:j} = a_i,\ldots,a_j$.
  Now, for every $a \in \X^n$,
  $\measure(a) = \prod_{i\in [n]} \measure(a_i \mid a_{1:i-1}$.
  For start, we show that
  $\measure(a_2\mid a_1)=\measure(a_2)$. Indeed,
  \begin{align*}
      \measure(a_2\mid a_1) = \sum_{a_3,\ldots,a_n}\measure(a_2\mid a_{1},a_{3:n})\cdot \measure(a_{3:n}\mid a_1) 
      =\sum_{a_3,\ldots,a_n}\measure(a_2)\cdot \measure(a_{3:n}\mid a_1) 
      = \measure(a_2)
  \end{align*}
  In the same way it can be shown that $\measure(a_3)=\measure(a_3\mid a_{1:2})$ and so on.

\end{proof}

\subsection{Proof of Theorem~\ref{thm:dp-high-prob-bound}}

\begin{proof}[Proof of Theorem ~\ref{thm:dp-high-prob-bound}]
  Fix a measure $\measure$ on $\X^n$ with Gibbs-dependence $\psi_n$, and fix an $(\varepsilon,\delta)$-differentially private algorithm that takes a sample $S\in\X^n$ and returns $k$ predicates $h_1,\dots,h_k:\X\rightarrow\{0,1\}$. Assume towards contradiction that
  \begin{equation}
    \label{eq:3}
    \Pr_{S,\alg(S)}\left[
      \max_{i\in [k]}\abs{h_i(\measure)-h_i(S)}\geq 10\varepsilon + 2\mixing
    \right]
    \geq \frac{\delta}{\varepsilon}.
  \end{equation}

\begin{algorithm}[tb]
   \caption{Auxiliary Algorithm $\alg'$}
    \label{alg:contra-alg}
\begin{algorithmic}
   \STATE {\bfseries Input:} $\vec{S}=(S_1,\ldots,S_T)$, where $T=\frac{\varepsilon}{\delta}$.

\STATE $F \leftarrow \emptyset$ 
   \FOR{$t \in [T]$}
   \STATE $(h^t_1,\ldots,h^t_k) \leftarrow \alg(S_t)$
   \STATE $H_t \leftarrow \{(h^t_1,t),\ldots,(h^t_k,t)\}$
   \STATE $\bar{H_t} \leftarrow \{1-h\mid h\in H_t\}$
   \STATE $F \leftarrow F \cup H_t \cup \bar{H_t}$
   \ENDFOR
   \STATE Sample $(h^*,t^*)$ from $F$ using the exponential mechanism. Specifically, sample $(h^*,t^*)\in F$ with probability proportional to $\exp\big(\frac{\varepsilon n}{2}\left(h^*(S_{t^*})-h^*(\measure)\right)\big)$.
   \STATE Return $(h^*,t^*)$
\end{algorithmic}
\end{algorithm}

  Consider the procedure
  described in Algorithm~\ref{alg:contra-alg}. As differential private algorithms are immune to post-processing
  and by the composition theorem,
  $\alg'$ is by itself $(2\varepsilon, \delta)$-differentially private. 
  Given a multi-set $\vec{S}$ sampled from $\measure^T$,
  by \eqref{eq:3} we get that 
  \[\forall t: \Pr_{S_t,\alg(S_t)}\left[
      \max_{i\in [k]}\abs{h^t_i(\measure)-h^t_i(S_t)}\geq 10\varepsilon + 2\mixing
    \right] \geq \frac{\delta}{\varepsilon},\]
  and hence, by setting $T=\frac{\varepsilon}{\delta}$, we have that 
  \begin{align*}
    \Pr_{\vec{S},\alg'(\vec{S})} & 
    \left[
      \max_{t\in [T],i\in [k]}\abs{h^t_i(\measure)-h^t_i(S_t)} 
      \geq 10\varepsilon + 2\mixing
    \right] 
    \geq 
    1 - \left(1-\frac{\delta}{\varepsilon}\right)^T\geq \uf{2}.
  \end{align*}
  By Markov's inequality,
  \[\E_{\vec{S},\alg'(\vec{S})}\left[
      \max_{t\in [T],i\in [k]}\abs{h^t_i(\measure)-h^t_i(S_t)}
    \right]
    \geq 5\varepsilon + \mixing.\]
    Now
  the set constructed in the algorithm's run, $F$,  contains also the negation of each predicate,
  and hence
  \begin{align*}
    &\E_{\vec{S},\alg'(\vec{S})}
    \left[
      \max_{(h,t)\in F}\big\{h(S_t) - h(\measure)\big\}
    \right] 
    =\E_{\vec{S},\alg'(\vec{S})}\left[
      \max_{t\in [T],i\in [k]}\abs{h^t_i(\measure)-h^t_i(S_t)}
    \right] 
    \geq 5\varepsilon + \mixing      .
  \end{align*}

By the properties of the exponential mechanism
  (see \citet{mcsherry2007mechanism} or \citet{bassily2016algorithmic}), denoting the output of the algorithm by $(h^*,t^*)$ we get that
  \begin{align*}
    &\E_{(h^*,t^*)}
    \left[
    h^*(S_{t^*}) - h^*(\measure)
    \right] 
    \geq
    \max_{(h,t)\in F} \{ h^*(S_{t^*}) - h^*(\measure) \}
    - \frac{2}{\varepsilon n}\log(2Tk).
  \end{align*}
  Taking expectation on both sides yields
  \begin{align*}
    &\E_{\vec{S},\alg'(\vec{S})} 
    \left[
    h^*(S_{t^*}) - h^*(\measure)
    \right] 
    \geq
    \E_{\vec{S},\alg'(\vec{S})}\left[
    \max_{(h,t)\in F} \{ h^*(S_{t^*}) - h^*(\measure) \}
    \right]
    - \frac{2}{\varepsilon n}\log(2Tk) 
    \geq
    5\varepsilon + \mixing - \frac{2}{\varepsilon n}\log(2k\varepsilon/\delta).
  \end{align*}
For $n \geq \frac{\log(2k\varepsilon/\delta)}{\varepsilon^2}$,
  this is at least $2\varepsilon + \mixing$ which contradicts Lemma~\ref{lem:dp-expect-bound}.
\end{proof}
\end{document}